\documentclass[a4paper,fleqn]{cas-sc}

\usepackage[authoryear,longnamesfirst]{natbib}

\usepackage{url}
\usepackage{array}
\usepackage{amsmath}
\usepackage{amssymb}
\usepackage{amsthm}
\usepackage{algorithm}
\usepackage{algorithmic}
\usepackage{xcolor}
\usepackage[caption=false,font=normalsize,labelfont=sf,textfont=sf]{subfig}
\usepackage{multirow}

\def\tsc#1{\csdef{#1}{\textsc{\lowercase{#1}}\xspace}}
\tsc{WGM}
\tsc{QE}

\usepackage{amsmath,amsfonts,bm}

\def\Figref#1{Figure~\ref{#1}}

\def\Secref#1{Section~\ref{#1}}

\def\eqref#1{equation~\ref{#1}}
\def\Eqref#1{Equation~\ref{#1}}

\def\Algref#1{Algorithm~\ref{#1}}

\def\1{\bm{1}}

\def\vone{{\bm{1}}}

\def\vs{{\bm{s}}}

\def\mA{{\bm{A}}}

\def\mD{{\bm{D}}}
\def\mE{{\bm{E}}}

\def\mH{{\bm{H}}}

\def\mK{{\bm{K}}}

\def\mM{{\bm{M}}}

\def\mP{{\bm{P}}}
\def\mQ{{\bm{Q}}}

\def\mV{{\bm{V}}}
\def\mW{{\bm{W}}}
\def\mX{{\bm{X}}}

\def\mZ{{\bm{Z}}}

\DeclareMathAlphabet{\mathsfit}{\encodingdefault}{\sfdefault}{m}{sl}
\SetMathAlphabet{\mathsfit}{bold}{\encodingdefault}{\sfdefault}{bx}{n}

\def\gE{{\mathcal{E}}}

\def\gG{{\mathcal{G}}}

\def\gN{{\mathcal{N}}}

\def\gV{{\mathcal{V}}}

\def\sR{{\mathbb{R}}}
\def\sS{{\mathbb{S}}}

\newcommand{\R}{\mathbb{R}}

\newcommand{\softmax}{\mathrm{softmax}}

\DeclareMathOperator*{\argmin}{arg\,min}

\DeclareGraphicsExtensions{.png}
\newtheorem{theorem}{Theorem}
\newtheorem{assumption}{Assumption}
\def\Appref#1{Appendix~\ref{#1}}
\def\cite#1{\citep{#1}}

\begin{document}
\let\WriteBookmarks\relax
\def\floatpagepagefraction{1}
\def\textpagefraction{.001}

\shorttitle{Neighbourhood Transformer}    

\shortauthors{}  

\title [mode = title]{Neighbourhood Transformer: Switchable Attention for Monophily-Aware Graph Learning}  

\tnotemark[1] 

\tnotetext[1]{} 

%


\author[1]{Yi Luo}
\author[2,1]{Xu Sun}
\author[3]{Guangchun Luo}
\author[1,4]{Aiguo Chen}
\cormark[4]
\affiliation[1]{organization={Laboratory of Intelligent Collaborative Computing, University of Electronic Science and Technology of China}}
\affiliation[2]{organization={School of Computer Science and Engineering, University of Electronic Science and Technology of China}}
\affiliation[3]{organization={Ubiquitous Intelligence and Trusted Services Key Laboratory of Sichuan Province}, country={China}}

\cortext[1]{Corresponding author: Aiguo Chen}
\fntext[1]{Email: \url{agchen@uestc.edu.cn}}


\begin{abstract}
Graph neural networks (GNNs) have been widely adopted in engineering applications such as social network analysis, chemical research and computer vision.
However, their efficacy is severely compromised by the inherent homophily assumption, which fails to hold for heterophilic graphs where dissimilar nodes are frequently connected.
To address this fundamental limitation in graph learning, we first draw inspiration from the recently discovered monophily property of real-world graphs, and propose Neighbourhood Transformers (NT), a novel paradigm that applies self-attention within every local neighbourhood instead of aggregating messages to the central node as in conventional message-passing GNNs.
This design makes NT inherently monophily-aware and theoretically guarantees its expressiveness is no weaker than traditional message-passing frameworks.
For practical engineering deployment, we further develop a neighbourhood partitioning strategy equipped with switchable attentions, which reduces the space consumption of NT by over 95\% and time consumption by up to 92.67\%, significantly expanding its applicability to larger graphs.
Extensive experiments on 10 real-world datasets (5 heterophilic and 5 homophilic graphs) show that NT outperforms all current state-of-the-art methods on node classification tasks, demonstrating its superior performance and cross-domain adaptability.
The full implementation code of this work is publicly available at \url{https://github.com/cf020031308/MoNT} to facilitate reproducibility and industrial adoption.
\end{abstract}


\begin{highlights}
    \item Our method captures the newly explored monophily patterns in graphs.
    \item Provides compatibility with conventional homophilic-based approaches.
    \item Proposes neighbourhood partitioning, vastly reducing memory and time.
    \item Shows state-of-the-art node classification accuracy on diverse datasets.
\end{highlights}


\maketitle

\section{Introduction}\label{sec:intro}

Graph neural networks (GNNs) have emerged as a fundamental technology in the realm of graph learning, garnering extensive research interest and a wealth of practical applications over the past decade~\cite{DUC2026113682}.
Their versatility has been demonstrated across a wide array of disciplines.
In the domain of social network analysis, for instance, GNNs are utilized to predict user interactions and to pinpoint pivotal influencers within the network~\cite{AN2026114058}.
Within the field of patent analysis, GNNs are employed to model intricate semantic and structural relationships for intellectual property knowledge discovery~\cite{ZENG2026113959}.
In the realm of intelligent transportation, GNNs are adopted to model spatial-temporal dependencies among road segments and achieve accurate traffic flow prediction~\cite{WU2026114236}.
Central to the functionality of GNNs is the message passing paradigm, which facilitates the exchange of information between nodes and their adjacent neighbours, thereby allowing GNNs to harness both node features and the structural topology of the graph~\cite{MP}.

Message passing (MP) implicitly posits that adjacent nodes are relevant or similar to one another, as is often the case in social networks where connected individuals tend to share similar interests~\cite{Homo1, Homo2, Homo3}.
However, recent investigations have called into question this homophily assumption by introducing a collection of heterophilic benchmarks and scenarios, where the premise of similarity between neighbouring nodes does not consistently apply~\cite{Geom-GCN, HetGraph, GAT-sep}.
For instance, in financial transaction networks, the majority of users with whom fraudsters engage in transactions are not engaged in fraudulent activities themselves~\cite{LU2026113709}.
On such heterophilic graphs, researchers have noted that a node often exhibits similarity not with its immediate neighbours, but with its 2-hop neighbours, or the neighbours of its neighbours~\cite{HetGNN,H2GCN}.
This characteristic, referred to as monophily~\cite{Monophily1, Monophily2}, is also prevalent in homophilic graphs~\cite{EvenNet, GraphACL}.
Consequently, the monophily assumption appears to be a universal trait in graphs, irrespective of their degree of homophily.

\begin{figure*}[!t]
\centering
\subfloat[]{
    \includegraphics[page=1,width=0.25\textwidth]{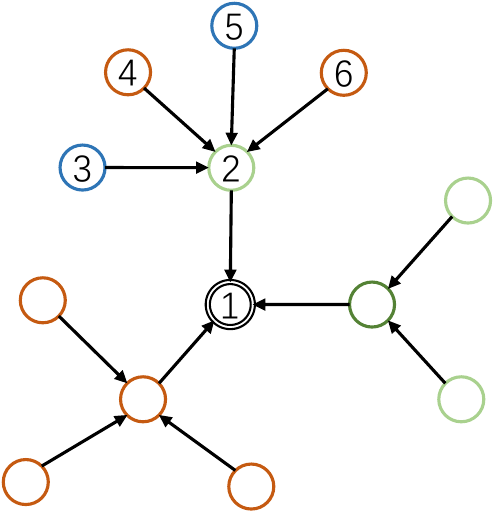}
    \label{fig:mp}}
\hfil
\subfloat[]{
    \includegraphics[page=2,width=0.25\textwidth]{figures.pdf}
    \label{fig:gt}}
\hfil
\subfloat[]{
    \includegraphics[page=3,width=0.25\textwidth]{figures.pdf}
    \label{fig:nt}}
\caption{
    An illustration of computing representations for node $v_1$ (represented by a double-lined circle) using various mechanisms.
    \textbf{(a) Message Passing}: Messages from different nodes (e.g., $v_3, v_4, v_5, v_6$) are propagated along edges towards node $v_1$ to compute its representation, but are diluted and over-squashed when passing through heterophilic bottlenecks (e.g., $v_2$).
    \textbf{(b) Graph Transformers}: Node $v_1$ aggregates messages from all nodes using self-attention. This approach may inadvertently downplay the importance of information from nearby nodes due to a reduction in the influence of topological structure.
    \textbf{(c) Neighbourhood Transformers (Ours)}: Node $v_1$ exchanges messages (depicted as squares) in each of its constituent neighbourhoods through self-attention, thereby acquiring attentive and structure-aware representations.
}
\label{fig:mpnt}
\end{figure*}

Drawing upon the monophily assumption, we introduce Neighbourhood Transformers (NT), which enables message exchanging among nodes within each neighbourhood through self-attention~\cite{Transformer} and constructs node representations by aggregating the exchanged messages from the neighbourhoods of all its neighbours, as \Figref{fig:nt} depicts.
While NT captures monophilc patterns in graphs, we prove that it is also compatible with MP.
Thus NT can serve as a replacement of traditional MP with greater or equal expressiveness.

\begin{figure*}[!t]
\centering
\includegraphics[width=\textwidth]{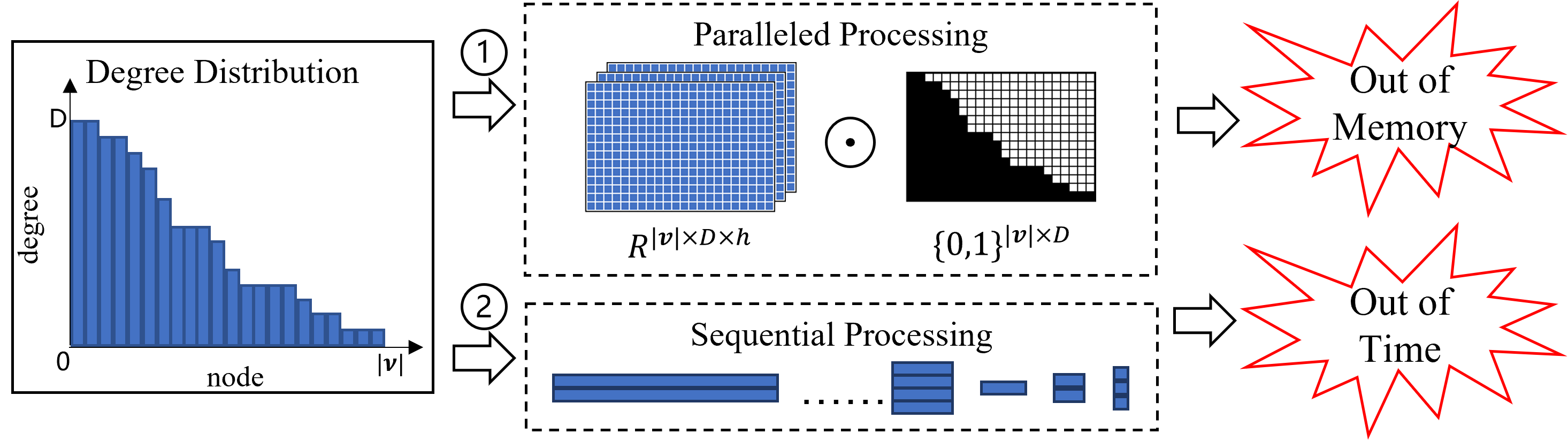}
\caption{
    Unevenly distributed neighbourhood sizes.
    \textcircled{1} \textbf{Paralleled Processing} pads node features of all neighbourhoods and processes the padded tensor in a single operation, occupying excessive memory when the node degree distribution is long-tailed.
    \textcircled{2} \textbf{Sequential Processing} organizes neighbourhoods by size and processes node features group by group, consuming prohibitive time when the node degree distribution is dispersed.
}
\label{fig:neidist}
\end{figure*}

A pivotal challenge of applying self-attention in each neighbourhoods is the high complexity resulting from the variable sizes of neighbourhoods.
In real-world graphs, the distribution of node degrees, correspond to the neighbourhood sizes, tends to be scattered and follow a long-tailed pattern, characterized by a small number of nodes with high degrees and a large number of nodes with low degrees~\cite{Long-tail}.
This attribute presents NT with a quandary, as illustrated in \Figref{fig:neidist}: either pad an excessive amount of redundant space to facilitate parallel processing~\cite{Transformer} or tolerate considerable time overhead to implement sequential processing~\cite{TinyGNN}.
Moreover, the quadratic complexity of self-attention must also be taken into account when processing large neighbourhoods.
To tackle these challenges, we incorporate an efficient variant~\cite{LinearAttention} of self-attention and devise a neighbourhoods partitioning strategy, significantly diminishing both the space and time requirements in NT.
For instance, when applied to the Tolokers dataset~\cite{GAT-sep}, our method reduces the memory footprint from over 80GB, necessitated by parallel processing, to less than 4GB and accomplishes the training in only 7.33\% of the time required for sequential processing.
These optimizations render NT practical, enabling us to conduct extensive experiments and assess its superiority against state-of-the-art baselines.

In summary, our contributions include \textbf{1)} a model, Neighbourhood Transformers, designed to harness the recently identified property of monophily within real-world graphs, \textbf{2)} a neighbourhood partitioning strategy equipped with switchable attentions to reduce space and time consumtions of NT, and \textbf{3)} extensive experiments across diverse benchmark graphs and thorough ablation studies.

\section{Related Works}

\textbf{Heterophlilic GNNs} (HGNN) are improved GNNs to address the challenges posed by heterophily~\cite{Geom-GCN, GloGNN, FAGCN, GBK-GNN, JacobiConv}.
Numerous prominent strategies of HGNN have been proposed in recent years.
For example, SPS-GAD~\cite{ZHU2026129639} segregates the graph into homophilic, ambiguous, and heterophilic subgraphs for different spectral filters.
Despite these orthogonal approaches which rewire the graph~\cite{DBLP:conf/uai/YangM24, DBLP:conf/www/WoS0GL24}, a common way to bypass the heterophilic structures is to pass messages along the second-order adjacency matrices $\mA^2$~\cite{EvenNet, H2GCN, GraphACL}.
While similar in principle to the utilization of $\mA^2$, Neighbourhood Transformer (NT) distinguishes itself by enabling nodes to directly access their 2-hop neighbours.
This paradigm effectively filters out potentially detrimental information in heterophilic contexts.
For example, consider the heterophilic neighbourhood of node $v_2$ in \Figref{fig:mp}, which contains node $v_1$ and from $v_3$ to $v_6$.
Assuming that nodes are mutually beneficial if they share the same color but are detrimental if they differ, conventional message passing (even when employing $\mA^2$) squashes all conflicting information at node $v_2$, subsequently propagate these noisy messages to its neighbouring nodes.
In contrast, NT (\Figref{fig:nt}) allows nodes (e.g. $v_6$) to safely access information from same-colored nodes (e.g. $v_4$) and filter out extraneous noise with the attention mechanism.
Another strategy of HGNNs is to aggregate neighbour- or high-pass-embeddings in separation with the ego- or low-pass-embeddings~\cite{H2GCN, GAT-sep, ACMGNN, FAGCN, ALT}.
Although effective in HGNNs, this strategy is not essential for NT, because NT separately treats 1-hop neighbouring nodes as context conditions and aggregates information from 0- or 2-hop monophilic nodes with similarity, thus obviating the need for such separation.
Our experiments in appendices has demonstrated this property of NT.

\textbf{Graph Transformers} (GT) employ self-attention across the entire node set to capture global data dependencies~\cite{DBLP:conf/nips/WuJWMGS21, NAGphormer}, as depicted in \Figref{fig:gt}.
This capability is advantageous in overcoming the information bottleneck associated with message passing (MP)~\cite{Over-squashing} and in aggregating high-order information to address heterophily challenges~\cite{DBLP:conf/nips/YingCLZKHSL21}.
However, this often prioritizes distant nodes due to the lack of topological regulation, potentially overlooking nearby nodes that typically carry more relevant information~\cite{Over-globalizing}.
As a trade-off, GT necessitates the explicit integration of structural encodings~\cite{SAN, RWSE} and the implicit representations of MP to address this shortcoming~\cite{Polynormer}.
These limitations suggest that GT should be integrated with MP~\cite{Exphormer, Ma2023GraphIB, GraphGPS}, rather than replacing MP as a standalone graph learning method.
In contrast to existing methods of GT, which are essentially the combination of MP and GT, NT is an improved replacement of MP that leverages local structures without succumbing to the issue of over-globalization and can also be integrated with GT.
Therefore, we compare NT against MP in the main text, as they are direct counterparts.
A comparison between NT and GT is additionally provided in \Appref{app:gt} for completeness.

\section{Preliminaries}

We denote a graph as $\gG = (\gV, \gE)$, where $\gV = \{ v_i | i = 1, 2, \ldots, |\gV| \}$ is the node set and $\gE = \{ e_{ij} | v_i \text{ connects } v_j \}$ is the edge set.
Node features are a matrix $\mX \in \sR^{|\gV| \times d}$ where $d$ is the dimensions of node features.
The $i$-th row $\mX_{i,:}$ corresponds to the feature vector of node $v_i$.
Edge attributes are a matrix $\mE \in \sR^{|\gE| \times d_e}$ where $d_e$ is the dimensions of edge attributes.
We denote the attributes of edge $e_{ij}$ as $\mE_{(i,j),:}$.
The neighbourhood of node $v_i$ is the set of nodes that connect to $v_i$, denoted as $\gN(v_i) = \{ v_j | e_{ij} \in \gE \}$.

\subsection{Message Passing in Graph Neural Networks}

The predominant architecture of graph neural networks (GNNs) is founded on the message passing (MP) mechanism, which comprises two essential components: the combiner and the aggregator.
The combiner is a node-specific function, such as a straightforward linear transformation or a multi-layer perceptron (MLP), that encodes the input node features $\mX$ into messages $\mZ \in \R^{|\gV| \times h}$.
The aggregator is an order-invariant operation, like $\mathrm{mean}$~\cite{GCN}, $\mathrm{max}$~\cite{SAGE}, $\mathrm{sum}$~\cite{GIN}, weighted-$\mathrm{mean}$~\cite{GAT, GATv2}, or gated-$\mathrm{sum}$~\cite{GatedGCN}, which is utilized to aggregate the messages from neighbouring nodes to produce the final node representations $\mH \in \R^{|\gV| \times h}$.
In essence, an MP layer operates as follows:
\begin{align*}
    \mZ &= \phi(\text{Combiner}(\mX)), \\
    \mH_{i,:} &= \text{Aggregator}(\{ \mZ_{j,:} | v_j \in \gN(v_i) \}),
\end{align*}
where $\phi(\cdot)$ is a non-linear activation function like GELU~\cite{GELU}.

\subsection{Self-attention in Transformers}

Self-attention is the core innovation of Transformer, designed to capture intricate dependencies among the $n$ input nodes.
It initially maps the node features $\mX \in \R^{n \times d}$ into corresponding query, key, and value matrices $\mQ, \mK, \mV \in \R^{n \times h}$.
Subsequently, an $n \times n$ correlation matrix is generated by the matrix multiplication of $\mQ$ and $\mK^T$, which indicates the importance weights for information selection from $\mV$.
Concisely, a self-attention module can be formalized as:
\begin{align*}
    (\mQ, \mK, \mV) &= \mX \cdot (\mW_q, \mW_k, \mW_v), \\
    \text{SelfAttention}(\mX) &= \softmax(\frac{\mQ \cdot \mK^T}{\sqrt{h}}) \cdot \mV,
\end{align*}
where $\mW_q, \mW_k, \mW_v \in \R^{d \times h}$ are trainable parameters.

\subsection{Linear-attention in Performers}\label{sec:perf}

The quadratic complexity $O(n^2)$ of self-attention, as evidenced in the operation $\mQ \cdot \mK^T$, becomes computationally infeasible when the number $n$ of nodes is substantial.
Consequently, a variety of efficiency-enhanced attention mechanisms with linear complexity $O(n)$ have been proposed~\cite{LinearAttention}.
Among these, Performer~\cite{Performer} offers an unbiased or nearly-unbiased estimation of self-attention, complete with provable convergence and reduced variance.
It initially maps the query and key matrices $\mQ, \mK$ into $\hat\mQ, \hat\mK \in \R^{n \times p}$ using its orthogonal random features $\mP \in \R^{h \times p}$.
Next, the product of $\hat\mK^T$ and $\mV$ results in a $p \times h$ matrix, which then matrix-multiplies $\hat\mQ$ to yield the approximated attention weights.
The approximated self-attention mechanism in Performer is expressed as:
\begin{align*}
    \hat\mQ &= \exp( \frac{1}{\sqrt{h}} \cdot \mQ \cdot \mP ), \\
    \hat\mK &= \exp( \mK \cdot \mP - \frac{\| \mK \|^2}{2} ), \\
    \hat\mD &= \text{diag}(\hat\mQ \cdot (\hat\mK^T \cdot \vone_{n \times 1})), \\
    \text{SelfAttention}(\mX) &\approx \hat\mD^{-1} \hat\mQ \cdot (\hat\mK^T \cdot \mV).
\end{align*}
When $h$ is held constant and $p$ is set to $h \log h$, as recommended by Performer, the complexity of this linear-attention is $O(nph) = O(n)$.

\section{Neighbourhood Transformers}

Motivated by the monophily in graphs, we propose Neighbourhood Transformer (NT) to facilitate the message exchanging among nodes from the same neighbourhoods:
\begin{align}
    \mZ_{(j,k),:} &= \phi(\text{Combiner}(\begin{bmatrix} \mH'_{j,:} & \mH'_{k,:} \end{bmatrix})) \label{eq:comb}, \\
    \mM^{(j)}     &= \phi(\text{SelfAttention}(\mathbin{\oplus} \{ \mZ_{(j,k),:} | v_k \in \gN(v_j) \})) \label{eq:att}, \\
    \mH_{i,:}     &= \text{Aggregator}(\{ \mM^{(j)}_{(i),:} | v_j \in \gN(v_i) \}) \label{eq:agg},
\end{align}
where $\mathbin\oplus$ denotes the operation of stacking a set of row vectors into a matrix, $\mH'$ is the node representations from the previous NT layer initialized as $\mH' = \mX$.
In \Eqref{eq:comb}, the initial message $\mZ_{(j,k),:}$ from node $v_j$ to node $v_k$ is computed by combining their node representations $\mH'_{j,:}$ and $\mH'_{k,:}$.
In \Eqref{eq:att}, the matrix $\mM^{(j)} \in \R^{|\gN(v_j)| \times h}$ encapsulates messages from node $v_j$ to all its neighbours, which are comprehensively exchanged within its neighbourhood $\gN(v_j)$ through self-attention
This message exchanging exploits the monophily property, a feature that has been shown to be advantageous in heterophilic graphs, and aids nodes in capturing intricate dependencies among similar 2-hop neighbours.
In \Eqref{eq:agg}, NT computes the $h$-dimensional representation $\mH_{i,:}$ for node $v_i$ by aggregating messages $\mM^{(j)}_{(i),:}$ from all its adjacent nodes $v_j \in \gN(v_i)$.
Here, $\mM^{(j)}_{(i),:}$ corresponds to the row vector for node $v_i$ in matrix $\mM^{(j)}$.

\subsection{Theoretical Analysis}

The combination of both nodes in \Eqref{eq:comb} is pivotal in NT because the message exchanged within each neighbourhood is thus uniquely tailored to its central node, endowing NT with the diversified capability to accommodate various graphs.
As the following theorems indicate, NT would degrade to mere message passing if any part of the two nodes is excluded:
\begin{theorem}\label{thm:1hop}
When the combiner concentrates on information from central nodes of neighbourhoods, the Neighbourhood Transformer is a message passing layer.
\end{theorem}
\begin{proof}
When \Eqref{eq:comb} omits $\mH'_{k,:}$ and becomes
\[\mZ_{(j,k),:} = \phi(\text{Combiner}(\mH'_{j,:})) \triangleq \mZ^{(j)},\]
\Eqref{eq:att} becomes a simple transformation of $\mZ^{(j)}$ as
\begin{align*}
    \mM^{(j)} &= \phi(\text{SelfAttention}(\oplus \{ \mZ^{(j)}, \mZ^{(j)}, \ldots, \mZ^{(j)} \} ) \\
    &= \begin{bmatrix} \phi(\mZ^{(j)} \mW_v) \\ \phi(\mZ^{(j)} \mW_v) \\ \cdots \\ \phi(\mZ^{(j)} \mW_v) \end{bmatrix}.
\end{align*}
Then, the output of \Eqref{eq:agg} is actually equivalent to the output of a message passing layer:
\begin{align*}
    \mH_{i,:} &= \text{Aggregator}({\mM^{(j)}_{(i),:} | v_j \in \gN(v_i)}) \\
    &= \text{Aggregator}({\phi(\phi(\text{Combiner}(\mH'_{j,:})) \cdot \mW_v) | v_j \in \gN(v_i)}).
\end{align*}
\end{proof}
\noindent This theorem also suggests that the optimizable combiner ensures compatibility with traditional message passing and maintains expressiveness not weaker than conventional GNNs when monophily is absent.

\begin{theorem}\label{thm:2hop}
When the combiner omits information from central nodes of neighbourhoods, the Neighbourhood Transformer with linear-attention is a two-layered message passing network.
\end{theorem}
\begin{proof}
When omitting $\mH'_{j,:}$ in \Eqref{eq:comb} and using linear-attention in \Eqref{eq:att}, the final representations of \Eqref{eq:agg} are
\[\mH_{i,:} = \text{Aggregator}( \{ \phi(\frac{\hat \mQ_{i,:} \cdot \mK_v^{(j)}}{\hat \mQ_{i,:} \cdot \mK_1^{(j)}}) | v_j \in \gN(v_i) \} ),\]
where $\mK_v^{(j)}$ and $\mK_1^{(j)}$ indicates the formulas $\hat \mK^T \cdot \mV$ and $\hat \mK^T \cdot \vone_{n \times 1}$ of Performer applied in the neighbourhood $\gN(v_j)$.
In detail, the element at the position $(x, y)$ of $\mK_v^{(j)}$ is $\sum\limits_{v_k \in \gN(v_j)} \hat K_{k,x} V_{k,y}$ and the $x$-th element of $\mK_1^{(j)}$ is $\sum\limits_{v_k \in \gN(v_j)} \hat K_{k,x}$.

Regarding $\mK_v^{(j)}$ and $\mK_1^{(j)}$ as the result of a message passing layer, which aggregates information of $\gN(v_j)$ to node $v_j$, NT can be rewritten as a two-layered message passing network:
\begin{align*}
\mZ^{(1)} &= (\hat \mK, \mV) = (\exp(\mX \mW_k \mP - \frac{||\mX\mW_k||^2}{2}), \mX\mW_v), \\
\mH^{(1)}_{i,:} &= (\mK_v^{(i)}, \mK_1^{(i)}) \\
    &= ([\sum\limits_{v_j \in \gN(v_i)} \hat K_{j,x} \cdot V_{j,y}]_{xy}, [\sum\limits_{v_j \in \gN(v_i)} \hat K_{j,x}]_x), \\
\mZ^{(2)} &= \mH^{(1)}, \\
\mH^{(2)}_{i,:} &= \text{Aggregator}'(\exp(\frac{1}{\sqrt{h}} \cdot \mX_{i,:} \mW_q \mP), \{ \mZ^{(2)}_{j,:} | v_j \in \gN(v_i) \}),
\end{align*}
where $\mX$ is the inputted node features.
\end{proof}

\subsection{Implementation Details}\label{sec:nt_details}

In addition to the static aggregators, such as $\mathrm{mean, max, sum}$, we harness the exchanged messages to simplify the implementation of dynamic aggregators.
Specifically, we double the dimensions of the output matrices $\mM^{(j)}$ from \Eqref{eq:att} and split them into two parts.
The first part is normalized by $\sigma$, which corresponds to the Softmax function for the weighted-$\mathrm{mean}$ aggregator and the Sigmoid function for the gated-$\mathrm{sum}$ aggregator, to generate the weights for aggregating the second part.
This results in \Eqref{eq:agg} being expressed as:
\begin{align*}
    \mH_{i,:} &= \text{Aggregator}(\{ \mM^{(j)}_{(i),:} | v_j \in \gN(v_i) \}) \\
    &= \sum\limits_{v_j \in \gN(v_i)} \sigma( \frac{1}{h} \cdot \sum\limits^{h}_{l=1} \mM^{(j)}_{(i),l}) \cdot \mM^{(j)}_{(i),h+1:2h}.
\end{align*}
The weights of the messages $\mM^{(j)}_{(i),h+1:2h}$ from node $v_j$ to node $v_i$ are determined not solely by the two endpoints but by the entire neighbourhood $\gN(v_j)$, thereby enhancing the attentiveness and suitability of our approach for NT.

\begin{figure}[!t]
\centering
\includegraphics[page=6,width=0.25\textwidth]{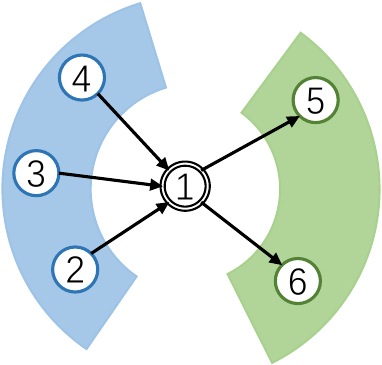}
\caption{
    In directed graphs, the source nodes ($v_2, v_3, v_4$) to the central node ($v_1$) and its destination nodes ($v_5, v_6$) construct different neighbourhoods.
}
\label{fig:dirnt}
\end{figure}

In the context of directed graphs, we adhere to the findings of \citet{Dir-GNN}, which suggest that leveraging the directionality of edges can lead to substantial improvements, particularly in heterophilic graphs.
To adapt NT for directed graphs, we introduce a directed variant called Dir-NT.
This variant differentiates between the source neighbours (nodes that have edges pointing towards the central node) and the destination neighbours (nodes that the central node points to), as shown in \Figref{fig:dirnt}.
We apply two separate NT instances to these distinct sets of neighbours.
The mathematical expression for this approach is as follows:
\begin{align*}
    \text{NT}_1(\mX, \mE) + \text{NT}_2(\mX, \mE'), \\
    \text{ where } \mE'_{(k, j),:} = \mE_{(j, k),:}, \forall e_{jk} \in \gE.
\end{align*}
Here, $\text{NT}_1$ and $\text{NT}_2$ represent two separate instances of the NT.
By summing the outputs of $\text{NT}_1$ and $\text{NT}_2$, we combine the information from both the source and destination neighbourhoods, allowing the model to capture the directional information in the graph and potentially improve the representation learning for nodes in directed graphs.

\section{Neighbourhood Partitioning}

\begin{figure*}[!t]
\centering
\includegraphics[page=5,width=\textwidth]{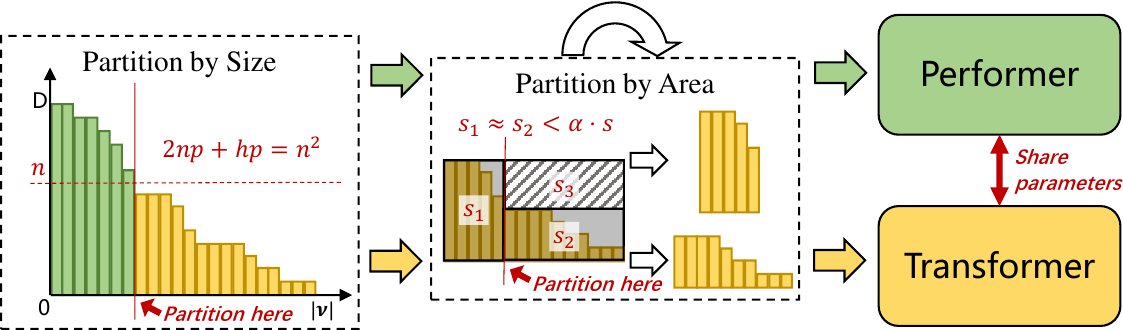}
\caption{
    Partitioning neighbourhoods into multiple groups for space and time efficiency.
    We partition neighbourhoods into two groups based on the degrees of their central nodes.
    Neighbourhoods with more than $n$ nodes will be processed by Performer for efficiency and other small neighbourhoods will be processed by Transformer for accuracy.
    Each group of neighbourhoods is recursively partitioned into two approximately equal halves in terms of area ($s_1 \approx s_2$), provided that the partitioning leads to a considerable compression rate $\alpha$.
}
\label{fig:section}
\end{figure*}

As we have discussed in \Secref{sec:intro}, the primary obstacle to the practical application of NT is the excessive complexity arising from the uneven distribution of neighbourhood scales.
To address this challenge, we propose a neighbourhood partitioning strategy that partitions all neighbourhoods into several smaller groups.
Each group is processed through an attention module that can be switched between the precise self-attention of the Transformer and the computationally efficient linear-attention of the Performer.
We depict this strategy in \Figref{fig:section} and provide a detailed description in the subsequent text.

\subsection{Partitioning by Size for Switchable Attention}

The self-attention mechanism, with its quadratic complexity, necessitates a significant amount of memory when applied to large neighbourhoods.
To mitigate this issue, we implement a switchable attention module that processes neighbourhoods with different sizes.
This module employs the linear-attention of the Performer to handle neighbourhoods exceeding $n$ nodes, prioritizing efficiency, and switches to the self-attention of the Transformer for neighbourhoods with $n$ nodes or fewer, ensuring accuracy.
Notably, both attention algorithms share a single set of parameters, which is made possible by the Performer's remarkable property of full compatibility with the Transformer.
As detailed in \Secref{sec:perf}, the self-attention in the Transformer computes an $n \times n$ matrix $\mQ \cdot \mK^T$, whereas the linear-attention in the Performer computes two $n \times p$ matrices $\hat\mQ, \hat\mK$ and a $p \times h$ matrix $\hat\mK^T \cdot \mV$.
Here, $p$ and $h$ represent the dimensions of the orthogonal random features in the Performer and the node representations, respectively.
Thus, we set $n = p + \sqrt{p^2 + hp}$ as the solution of $n^2 = 2np + hp$ to ensure that the linear-attention indeed reduces the memory footprint compared to the original self-attention.

\subsection{Recursively Partitioning by Area for Space-and-Time Efficiency}

In scenarios where we have $N$ neighbourhoods with the largest one contains $D$ nodes, processing the neighbourhoods in parallel necessitates padding the node features into an $N \times D \times d$ tensor (with an $N \times D$ boolean matrix to indicate the paddings).
Due to the long-tail distribution characteristic of real-world graphs, this pre-processing step often results in substantial space wastage due to padding and an increase in time due to redundant computation.
To address this issue, we propose to partition the neighbourhoods into smaller groups and process them sequentially rather than in a single paralleled operation.
To formalize our analysis, we outline the following two assumptions to measure the memory usage in neighbourhood processing.
\begin{assumption}[Paralleled Processing]
The memory consumption of applying the transformer to a group of neighbourhoods is proportional to the group's area, defined as the product of the number of neighbourhoods and their maximum size.
\end{assumption}
\begin{assumption}[Sequential Processing]
Sequentially processing two neighbourhood groups with areas $s_1$ and $s_2$ consumes memory proportional to $\max(s_1, s_2)$.
\end{assumption}

As the example in \Figref{fig:section} shows, partitioning neighbourhoods into two groups and sequentially processing them results in a memory footprint proportional to the area of $\max(s_1, s_2)$, which is smaller than the memory footprint of the original group, proportional to the area of $s = s_1 + s_2 + s_3$.
However, searching for the optimal partitioning with $\min \max(s_1, s_2)$ has a complexity of $o(2^N)$.
To accelerate the searching, we propose the following theorem:
\begin{theorem}\label{thm:section}
Given any partitioning that divides neighbourhoods into two groups, there exists an alternative partitioning that requires the same or less processing memory, where all neighbourhoods in one group are not smaller than those in the other group.
\end{theorem}
\begin{proof}
Considering a partitioning where group $G_1$ contains $c_1$ neighbourhoods and group $G_2$ contains $c_2$ neighbourhoods, with the maximum size $d_1$ of neighbourhoods in $G_1$ being not smaller than that $d_2$ of $G_2$ ($d_1 \ge d_2$).
If the smallest neighbourhood (with size $d_3$) in $G_1$ is smaller than the largest one (with size $d_2 > d_3$) in $G_2$, we can swap their positions to create two new groups $G'_1$ and $G'_2$.
The area of $G'_1$ remains unchanged since $\max(d_1, d_2) \times c_1 = d_1 \times c_1$, while the area of $G'_2$ is unchanged or reduced since $\min(d_3, d_4) \times c_2 \le d_2 \times c_2$, where $d_4$ is the size of the second-largest neighbourhood in $G_2$.
We can keep swapping the smallest neighbourhood in the first group with the largest neighbourhood in the second group if the former is smaller than the latter until any neighbourhood in the first group is not smaller than those in the second, with the processing memory remaining unchanged or reduced.
\end{proof}
\noindent With this theorem, we conclude that the optimal partitioning can be achieved by first ordering the neighbourhoods and then scanning linearly for the partitioning point, with a complexity of $o(N \log N)$, accounting for the sorting step.

\begin{algorithm}[tb]
    \caption{Recursively Partitioning Neighbourhoods by Area}\label{alg:section}
    \begin{algorithmic}[1]
        \REQUIRE compression rate $\alpha$, a list $[(n_1, c_1), (n_2, c_2), \ldots, (n_l, c_l)]$ with $n_1 > n_2 > \ldots > n_l$ and $\sum\limits_{i=1}^l c_i = |\gV|$ where an element $(n_i, c_i)$ indicates $c_i$ neigbourhoods with size $n_i$.
        \ENSURE a set $\sS$ where an element $(i, j, s)$ indicates a group of neighbourhoods sized in $[n_i, n_j]$ and its area is $s$.
        \STATE \textbf{Define:} $\text{Area}(i, j) = (a_j - a_i + c_i) \times n_i, a_i = \sum\limits_{j=1}^i c_j$
        \STATE \textbf{Initialize:} $\sS := \{ (1, l, \text{Area}(1, l)) \}$
        \WHILE{true}
            \STATE select $(i, j, s)$ from $\sS$ with maximum $s$
            \IF{$ j = i $}
                \RETURN $\sS$
            \ENDIF
            \STATE $t^* := \argmin\limits_{t=i}^j \max(\text{Area}(i, t), \text{Area}(t+1, j))$
            \STATE $s_1 := \text{Area}(i, t^*)$
            \STATE $s_2 := \text{Area}(t^*+1, j)$
            \IF{$\max(s_1, s_2) \ge \alpha \cdot s$}
                \RETURN $\sS$
            \ENDIF
            \STATE $\sS := \sS \setminus \{ (i, j, s) \} \cup \{ (i, t^*, s_1), (t^*+1, j, s_2) \}$
        \ENDWHILE
    \end{algorithmic}
\end{algorithm}

Although the sequential handling of the partitioned neighbourhood groups may increase the processing time, a substantial enough saving in redundant computation on padded bits (e.g. the $s_3$ part) can offset this time increase.
We therefore adopt the partitioning only when the compression rate $\alpha = \max(s_1, s_2) / s$ is considerably small and develop an algorithm, described in \Algref{alg:section}, to recursively partition neighbourhoods into multiple groups for both space and time efficiency.
In the algorithm, we define a function to calculate the area of a neighbourhood group as the product of the number of its containing neighbourhoods and their maximum size (in line 1).
Then, we initialize the algorithm with all inputted neighbourhoods as a single group (in line 2).
In the main loop, we repeatedly select the group with the largest area from the set $\sS$ (in line 4) and partition it into two halves with their maximum areas minimal (in line 8).
If the partitioning leads to a considerable compression rate (in line 11), we accept the partitioning and replace the group with its two halves (in line 14).
By adjusting the hyperparameter $\alpha$, we can control the tradeoff between memory usage and processing time.
Our empirical findings suggest that $\alpha = 0.4$ is a good balance, resulting in fast processing with relatively low memory consumption.
When a group is atomic (line 6) or a partitioning is refused (line 12), we terminate partitioning other smaller groups and output the group set $\sS$. 

\section{Experiments}\label{sec:exp}

To evaluate the performance of the Neighbourhood Transformer (NT) and the proposed neighbourhood partitioning strategy with switchable attention, we design a series of node classification experiments on a diverse set of graphs.
The first five datasets are heterophilic graphs in Table~\ref{tbl:het}, including Roman Empire, A-ratings, Minesweeper, Tolokers, and Questions~\cite{GAT-sep}.
The latter five datasets are homophilic graphs in Table~\ref{tbl:homo}, including A-computer, A-photo~\cite{Amazon}, CoauthorCS, CoauthorPhy~\cite{HomoGraph}, and WikiCS~\cite{WikiCS}.
We use the adjusted homophily metric, as introduced in \citet{GAT-sep}, to quantify the degree of homophily within a graph.
It is evident from the measurements that heterophilic graphs exhibit lower homophily scores across the board.
We use the default data splits provided with the original datasets for the five heterophilic graphs and WikiCS.
For other four homophilic graphs, we follow the splitting strategy from \citet{Exphormer}, dividing the nodes into training (60\%), validation (20\%), and testing (20\%) sets.

Our experimental setup involves the integration of NT into the GAT-sep architecture, as proposed in \citet{GAT-sep}.
Specifically, the network architecture is structured as follows: it begins with a linear encoder, followed by $L$ residual blocks, and concludes with a linear predictor.
Each residual block incorporates a skip connection~\cite{SkipConnection} and consists of a layer normalization layer, an NT layer, and a two-layer multi-layer perceptron (MLP).
For model training, we utilize the Adam optimizer~\cite{Adam}.

\subsection{Node Classification Performance}\label{sec:node}

\begin{table*}[!t]
\caption{
    Averaged accuracy scores and the standard deviations in 10 runs on heterophilic graphs.
    The best score for each dataset is \textbf{\textcolor{red}{bolded}}, and the second best is \underline{\textcolor{blue}{underlined}}.
}\label{tbl:het}
\begin{tabular}{lrrrrr}
\toprule
            & Roman Empire & A-ratings & Minesweeper & Tolokers & Questions\\
\#Nodes     & 22,662       & 24,492    & 10,000      & 11,758   & 48,921 \\
\#Edges     & 32,927       & 93,050    & 39,402      & 519,000  & 153,540 \\
Mean Deg.   & 2.91         & 7.60      & 7.88        & 88.28    & 6.28  \\
\#Features  & 300          & 300       & 7           & 10       & 301  \\
\#Classes   & 18           & 5         & 2           & 2        & 2  \\
Homophily   & -4.68\%      & 14.02\%   & 0.94\%      & 9.26\%   & 2.07\%  \\
\midrule
GCN  & 73.69±0.74 & 48.70±0.63 & 89.75±0.52 & 83.64±0.67 & 76.09±1.27 \\
GraphSAGE & 85.74±0.67 & \underline{\textcolor{blue}{53.63±0.39}} & 93.51±0.57 & 82.43±0.44 & 76.44±0.62 \\
GAT & 80.87±0.30 & 49.09±0.63 & 92.01±0.68 & 83.70±0.47 & 77.43±1.20 \\
H2GCN & 60.11±0.52 & 36.47±0.23 & 89.71±0.31 & 73.35±1.01 & 63.59±1.46 \\
CPGNN & 63.96±0.62 & 39.79±0.77 & 52.03±5.46 & 73.36±1.01 & 65.96±1.95 \\
GPRGNN & 64.85±0.27 & 44.88±0.34 & 86.24±0.61 & 72.94±0.97 & 55.48±0.91 \\
FAGCN & 65.22±0.56 & 44.12±0.30 & 88.17±0.73 & 77.75±1.05 & 77.24±1.26 \\
GloGNN & 59.63±0.69 & 36.89±0.14 & 51.08±1.23 & 73.39±1.17 & 65.74±1.19 \\
GBK-GNN & 74.57±0.47 & 45.98±0.71 & 90.85±0.58 & 81.01±0.67 & 74.47±0.86 \\
JacobiConv & 71.14±0.42 & 43.55±0.48 & 89.66±0.40 & 68.66±0.65 & 73.88±1.16 \\
GGCN & 74.46±0.54 & 43.00±0.32 & 87.54±1.22 & 77.31±1.14 & 71.10±1.57 \\
tGNN & 79.95±0.75 & 48.21±0.53 & 91.93±0.77 & 70.84±1.75 & 76.38±1.79 \\
OrderedGNN~ & 77.68±0.39 & 47.29±0.65 & 80.58±1.08 & 75.60±1.36 & 75.09±1.00 \\
GAT-sep & 88.75±0.41 & 52.70±0.62 & 93.91±0.35 & 83.78±0.43 & 76.79±0.71 \\
Dir-GNN & 91.23±0.32 & 47.89±0.39 & 87.05±0.69 & 81.19±1.05 & 76.13±1.24 \\
CDE & \underline{\textcolor{blue}{91.64±0.28}} & 47.63±0.43 & \underline{\textcolor{blue}{95.50±5.23}} & - & 75.17±0.99 \\
BloomGML & 85.26±0.25 & 52.92±0.39 & 93.30±0.16 & \textbf{\textcolor{red}{85.92±0.14}} & \underline{\textcolor{blue}{77.93±0.34}} \\
\textbf{NT} & \textbf{\textcolor{red}{94.77±0.31}} & \textbf{\textcolor{red}{54.25±0.50}} & \textbf{\textcolor{red}{97.42±0.50}} & \underline{\textcolor{blue}{85.69±0.54}} & \textbf{\textcolor{red}{78.46±1.10}} \\
\bottomrule
\end{tabular}
\end{table*}

\begin{table*}[!t]
\caption{
    Averaged accuracy scores and the standard deviations in 10 runs on homophilic graphs.
    The best score for each dataset is \textbf{\textcolor{red}{bolded}}, and the second best is \underline{\textcolor{blue}{underlined}}.
}\label{tbl:homo}
\begin{tabular}{lrrrrr}
\toprule
             & A-computer & A-photo & CoauthorCS & CoauthorPhy & WikiCS \\
\#Nodes      & 13,752     & 7,650   & 18,333     & 34,493      & 11,701 \\
\#Edges      & 245,861    & 119,081 & 81,894     & 247,962     & 216,123 \\
Mean Deg.    & 35.76      & 31.13   & 8.93       & 14.38       & 36.85 \\
\#Features   & 767        & 745     & 6,805      & 8,415       & 300  \\
\#Classes    & 10         & 8       & 15         & 5           & 10 \\
Homophily    & 68.23\%    & 78.50\% & 78.45\%    & 87.24\%     & 57.90\% \\
\midrule
GCN & 89.65±0.52 & 92.70±0.20 & 92.92±0.12 & 96.18±0.07 & 77.47±0.85 \\
GraphSAGE & 91.20±0.29 & 94.59±0.14 & 93.91±0.13 & 96.49±0.06 & 74.77±0.95 \\
GAT & 90.78±0.13 & 93.87±0.11 & 93.61±0.14 & 96.17±0.08 & 76.91±0.82 \\
APPNP & 90.18±0.17 & 94.32±0.14 & 94.49±0.07 & 96.54±0.07 & 78.87±0.11 \\
PPRGo & 88.69±0.21 & 93.61±0.12 & 92.52±0.15 & 95.51±0.08 & 77.89±0.42 \\
GCNII & 91.04±0.41 & 94.30±0.20 & 92.22±0.14 & 95.97±0.11 & 78.68±0.55 \\
GPRGNN & 89.32±0.29 & 94.49±0.14 & 95.13±0.09 & 96.85±0.08 & 78.12±0.23 \\
GGCN & 91.81±0.20 & 94.50±0.11 & \underline{\textcolor{blue}{95.25±0.05}} & \underline{\textcolor{blue}{97.07±0.05}} & 78.44±0.53 \\
tGNN & 83.40±1.33 & 89.92±0.72 & 92.85±0.48 & 96.24±0.24 & 71.49±1.05 \\
OrderedGNN & \underline{\textcolor{blue}{92.03±0.13}} & \underline{\textcolor{blue}{95.10±0.20}} & 95.00±0.10 & 97.00±0.08 & \underline{\textcolor{blue}{79.01±0.68}} \\
\textbf{NT} & \textbf{\textcolor{red}{92.61±0.63}} & \textbf{\textcolor{red}{96.12±0.39}} & \textbf{\textcolor{red}{96.07±0.32}} & \textbf{\textcolor{red}{97.32±0.11}} & \textbf{\textcolor{red}{80.04±0.61}} \\
\bottomrule
\end{tabular}
\end{table*}

We report the average accuracy scores of NT across 10 runs on five heterophilic graphs in Table~\ref{tbl:het} and five homophilic graphs in Table~\ref{tbl:homo}.
Baselines are state-of-the-art (SotA) message passing neural networks (MPNN), including GCN~\cite{GCN}, GraphSAGE~\cite{SAGE}, GAT~\cite{GAT}, H2GCN~\cite{H2GCN}, CPGNN~\cite{CPGNN}, GPRGNN~\cite{GPRGNN}, FAGCN~\cite{FAGCN}, GloGNN~\cite{GloGNN}, GBK-GNN~\cite{GBK-GNN}, JacobiConv~\cite{JacobiConv}, GGCN~\cite{GGCN}, tGNN~\cite{tGNN}, OrderedGNN~\cite{OrderedGNN}, GAT-sep~\cite{GAT-sep}, Dir-GNN~\cite{Dir-GNN}, CDE~\cite{CDE}, BloomGML~\cite{BloomGML}, APPNP~\cite{PPNP}, PPRGo~\cite{PPRGO}, and GCNII~\cite{GCNII}.
Their node classification scores on the 10 datasets are retrieved from previous works, including their original papers~\cite{CDE,BloomGML} and leaderboards of the respective datasets~\cite{GAT-sep, Polynormer}.
The training protocol is constrained to a maximum of 2500 epochs with the learning rate for the optimizer is set at $0.001$.
An early stopping mechanism is implemented to terminate training when there is no improvement in the validation set performance for 500 consecutive epochs.
The selection of other hyperparameters is facilitated by Optuna~\cite{Optuna}, which is used to perform a search over the following parameters: the aggregator type, including $\mathrm{mean}$, weighted-$\mathrm{mean}$, $\mathrm{sum}$, gated-$\mathrm{sum}$, and $\mathrm{max}$; the number of hidden dimensions per attention head, ranging from 8 to 64; the number of attention heads, ranging from 1 to 8; the number of layers, ranging from 1 to 5; and the dropout rate, which is searched within the set $\{0.1, 0.2, \ldots, 0.8\}$.
The optimal hyperparameters identified through this search are summarized in \Appref{app:params} and recorded in our released code.

Table~\ref{tbl:het} and Table~\ref{tbl:homo} reveal that NT outperforms the existing SotA MPNNs on 8 of the 10 graphs and ranks as the second best on the rest two datasets.
The results on the five heterophilic graphs provide strong evidence that NT is a powerful approach for node classification tasks on heterophilic graphs, where traditional MPNNs often struggle due to the lack of homophily.
Besides, the results on the five homophilic graphs confirm our earlier analysis that NT is capable of adapting to the homophily present in graphs, making it a robust framework that is general for graph representation learning.
Thus, our claim that NT's message exchanging paradigm is a novel and competitive graph learning approach compared with the message passing paradigm is validated.

\subsection{The Efficiency of Neighbourhood Partitioning}

\begin{figure*}[!t]
\centering
\includegraphics[width=\textwidth]{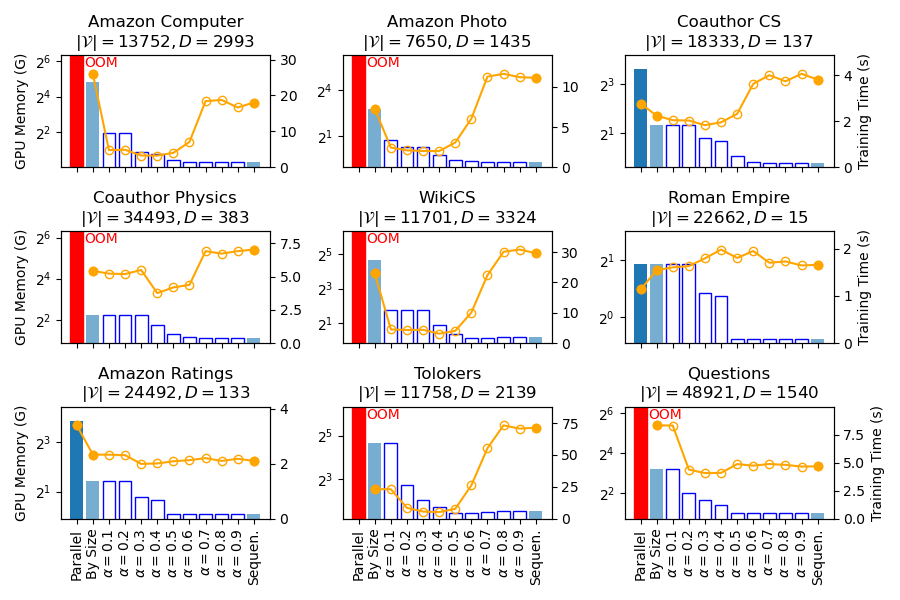}
\caption{
    Ablation studies on the neighbourhood partitioning strategy.
    Bars represent memory footprints and curves are time consumptions.
    We compare paralleled processing (the first bar/point), partitioning neighbourhoods by size (the second bar/point), by both size and area with $\alpha = 0.1, 0.2, \ldots, 0.9$ (the hollow bars/points), and sequential processing (the last bar/point).
}
\label{fig:ntmem}
\end{figure*}

\begin{figure*}[!t]
\centering
\includegraphics[width=\textwidth]{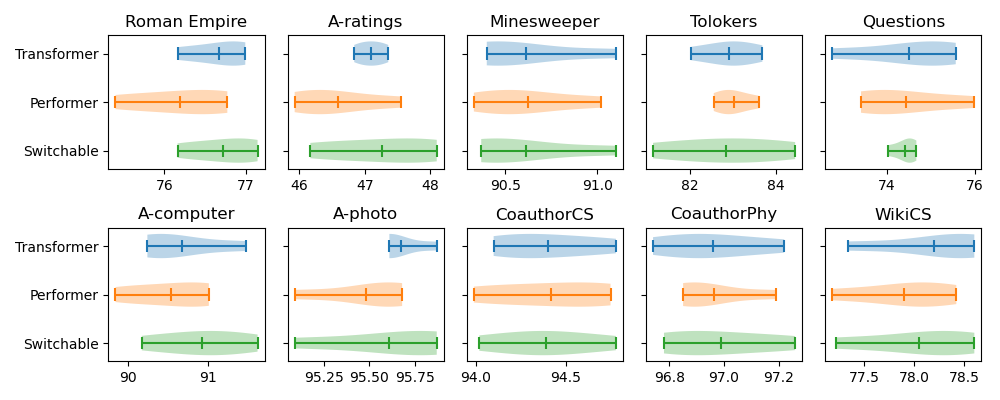}
\caption{
    Accuracy scores (\%, the horizontal axes) in 10 runs with different attention modules: Transformer with self-attention, Performer with linear-attention, and our switchable attention module.
    The middle line in each violin plot represents the average score.
}
\label{fig:frame}
\end{figure*}

In this section, we benchmark NT using nine datasets to elucidate the attributes of the neighbourhood partitioning strategy.
\Figref{fig:ntmem} provides a visual representation of the GPU memory consumption and training time for parallel processing, partitioning neighbourhoods by size, by both size and area, and sequential processing.
Parallel processing runs out of 80GB memory (OOM, depicted as red bars) in six out of the nine datasets and is only feasible on graphs with a low maximum node degree $D$.
By partitioning neighbourhoods by size and incorporating Performer, the memory footprint is dramatically decreased to below 30GB across all graph datasets.
With additional partitionings by area, the memory footprint is further reduced, for instance, to less than 4GB when $\alpha = 0.4$.
This reduction in GPU memory utilization is crucial for the practical application of NT.

Moreover, as illustrated in the figure, the training time curves exhibit a bowl-shaped pattern with the minimum points occurring around $\alpha = 0.4$, showing that our method can be an order of magnitude faster than sequential processing (e.g., 8.63 times faster on WikiCS and 12.64 times faster on Tolokers).
The only exception is observed on Roman Empire, where our approach is 16\% slower compared to sequential processing.
This discrepancy arises due to the highly concentrated distribution of node degrees in Roman Empire, which has an average node degree of 2.91.
In such a scenario, the simplicity of the graph structure allows sequential processing to handle all neighbourhoods within a limited number of operations, negating the benefits of our partitioning strategy.

As the above experiments indicate, the integration of our switchable attention module with the linear attention mechanism of Performer substantially diminishes the memory and computational demands of NT.
To investigate whether this integration impairs node classification accuracy, we conducted an ablation study focusing on the attention module.
\Figref{fig:frame} displays the accuracy scores for different attention modules across ten datasets.
As depicted in the figure, while the Performer with linear attention lags on Roman Empire, A-ratings, A-computer, A-photo, and WikiCS, there is no statistically significant discrepancy in accuracy between the Transformer with full-rank self-attention and our proposed switchable attention module.
Consequently, we deduce that switching between different attentions does not degrade classification accuracy.

In summary, our neighbourhood partitioning strategy successfully diminishes the space and time complexities associated with NT, all while maintaining its performance integrity.

\subsection{Additional Ablation Studies}

Other than \Figref{fig:frame} on different attention modules, we have conduct additional ablation studies to thoroughly reveal the property of Neighbourhood Transformers.

\subsubsection{On Aggregators}

\begin{figure*}[!t]
\centering
\includegraphics[width=0.95\textwidth]{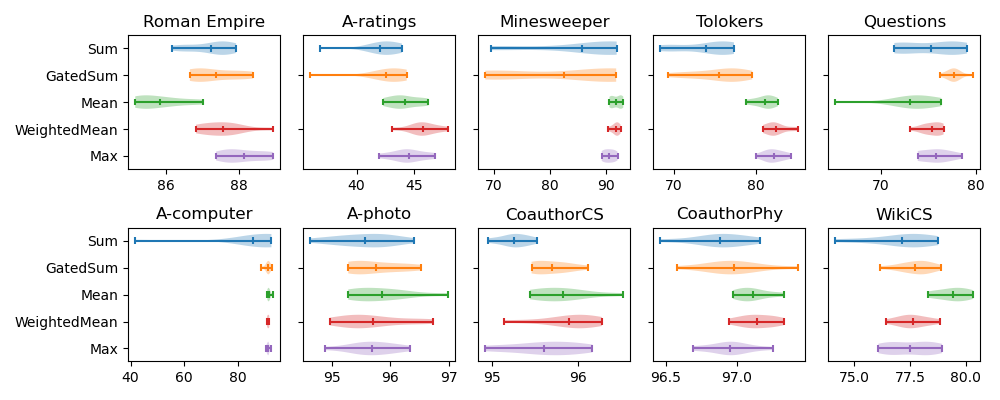}
\caption{
    Accuracy scores (\%, the horizontal axes) in 10 runs with different aggregators.
    The middle line in each violin plot represents the average score.
}
\label{fig:agg}
\end{figure*}

In this section, we perform an ablation study on the aggregator used within our NT.
We evaluate five different aggregators: $\mathrm{mean}$, weighted-$\mathrm{mean}$, $\mathrm{sum}$, gated-$\mathrm{sum}$, and $\mathrm{max}$.
As shown in \Figref{fig:agg}, there is no clear trend in performance across different aggregators, with the exception that the $\mathrm{sum}$ aggregator tends to be unstable and often results in worse performance.
Among the tested aggregators, weighted-$\mathrm{mean}$ appears to be a more robust choice overall.
However, the gated-$\mathrm{sum}$ aggregator achieves the highest score on the Questions dataset, while the $\mathrm{mean}$ aggregator performs best on the WikiCS dataset.
This suggests that the choice of aggregator can significantly impact the performance of NT and that the optimal aggregator may vary depending on the specific characteristics of the dataset.

\subsubsection{On Directed Graphs}

\begin{table}[!t]
\caption{
    Averaged accuracy scores and the standard deviations in 10 runs on heterophilic graphs.
}\label{tbl:dir}
\begin{tabular}{lccccc}
\toprule
     & Roman Empire & A-ratings & Minesweeper & Tolokers & Questions\\
\midrule
    GCN         & 73.69±0.74 & 48.70±0.63 & 89.75±0.52 & 83.64±0.67 & 76.09±1.27 \\
    Dir-GNN & 91.23±0.32 & 47.89±0.39 & 87.05±0.69 & 81.19±1.05 & 76.13±1.24 \\
    \textbf{NT}            & 91.71±0.57 & 54.25±0.50 & 97.42±0.50 & 85.69±0.54 & 78.46±1.10 \\
    \textbf{Dir-NT}        & 94.77±0.31 & 49.43±0.62 & 93.92±0.59 & 85.02±0.77 & 77.99±1.00 \\
\bottomrule
\end{tabular}
\end{table}

We report the performance of Dir-GNN~\cite{Dir-GNN} and the directed version of NT (Dir-NT) in Table~\ref{tbl:dir}, where the data clearly indicate that Dir-NT surpasses Dir-GNN across all tested datasets, with a particularly impressive accuracy score of 94.77 on the Roman Empire graph.
This demonstrates that Dir-NT is more effective at leveraging the directional information in edges compared to Dir-GNN.
However, we observe that on the A-ratings and Minesweeper datasets, both Dir-GNN and Dir-NT underperform compared to their undirected counterparts.
This discrepancy can be attributed to the nature of these datasets.
Although \citet{GAT-sep} has annotated these datasets with uni-directional edges, the relationships they represent, such as co-purchased products in A-ratings and adjacent grids in Minesweeper, are inherently undirected.
Consequently, modelling these graphs as directed does not provide any additional beneficial information.
Instead, it splits the neighbourhood into source and destination halves, which can interfere with the full exchange of monophilic messages.
Similarly, the Tolokers dataset, which represents a network of project colleagues, is also fundamentally undirected.
However, its significantly higher density (more than 10 times denser than the other four graphs) means that dividing the neighbourhood into two parts has a smaller negative impact on performance.

\subsubsection{On Embeddings Separation}\label{sec:sep}

\begin{figure*}[!t]
\centering
\includegraphics[width=\textwidth]{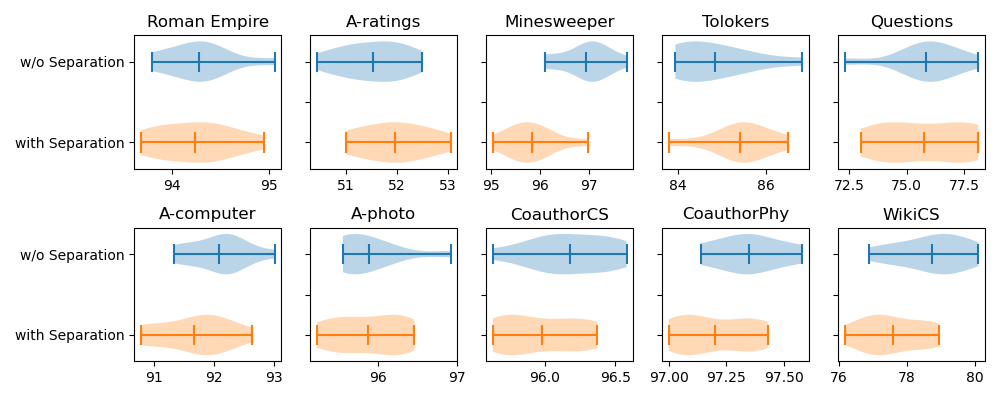}
\caption{
    Accuracy scores (\%, the horizontal axes) in 10 runs with and without the separation of ego- and neighbour-embeddings.
    The middle line in each violin plot represents the average score.
}
\label{fig:sep}
\end{figure*}

\citet{H2GCN} shows that the challenge posed by heterophily in graphs can be mitigated by distinguishing between ego-embeddings (representations of the central node itself) and neighbour-embeddings (representations of the node's neighbours) during the aggregation process.
Can this experience of message passing (MP) be brought into NT?
We answer this question by presenting \Figref{fig:sep}, which shows the outcomes of our study on whether to implement this separation in NT.
The results indicate that except for the A-ratings and Tolokers datasets, adopting the separation of ego- and neighbour-embeddings leads to a decline in performance on 8 out of 10 graphs.
The rationale behind this is that in heterophilic graphs, the central node tends to be dissimilar to its neighbouring nodes.
Therefore, these dissimilar nodes require different transformations before aggregating and the separation strategy addresses this requirement.
However, NT is to aggregate messages attentively from 0- and 2-hop neighbours, conditioned on 1-hop neighbours.
This inherently separation of ego-, 1-hop-neighbour, and 2-hop-neighbour-embeddings suggests that NT's inherent ability to handle the aggregation of node information may already be sufficient to capture the complex relationships in heterophilic graphs, rendering the separation of ego- and neighbour-embeddings an unnecessary step for improving performance in most cases.

\subsubsection{On Self-loops}

\begin{figure*}[!t]
\centering
\includegraphics[width=\textwidth]{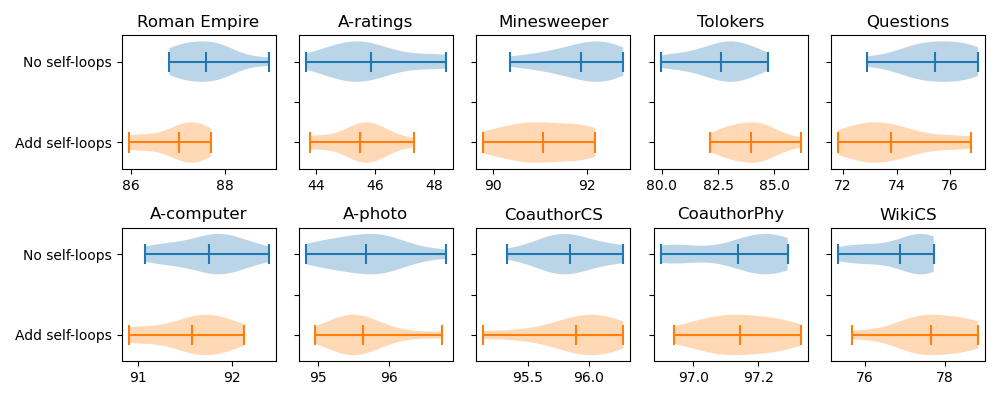}
\caption{
    Accuracy scores (\%, the horizontal axes) in 10 runs with and without self-loops.
    The middle line in each violin plot represents the average score.
}
\label{fig:selfloops}
\end{figure*}

In this section, we investigate the impact of adding self-loops to nodes in NT.
Adding self-loops, which are edges connecting a node to itself, is a technique often used to modify the graph spectrum and facilitate the learning of smoother representations~\cite{SGC}.
When self-loops are added in NT, each node effectively becomes part of its own neighbourhood.
This inclusion introduces an inductive bias towards homophily, as it assumes that nodes are similar to their neighbours, which may not always be the case in heterophilic graphs.
The ablation study presented in this section, as depicted in \Figref{fig:selfloops}, reveals that incorporating self-loops can occasionally degrade performance on heterophilic graphs due to the aforementioned incorrect inductive bias.
In contrast, for homophilic graphs, the addition of self-loops does not lead to an accuracy increase on 4 out of 5 datasets.
This suggests that NT is already adept at capturing the homophily present in these graphs, and thus, the extra self-loops do not contribute additional benefits.
These findings highlight the importance of considering the underlying graph structure and the nature of the relationships between nodes when deciding on the use of self-loops in graph representation learning models like NT.

\section{Conclusions}

We introduce Neighbourhood Transformers (NT) designed to exploit the monophily observed in real-world graphs.
NT employs self-attention mechanisms within each neighbourhood of the graph, generating informative messages for the nodes within, as opposed to the central node.
This exploitation allows NT to effectively address heterophily and to be adaptive in homophilic graphs.
We overcome the space and computational challenges inherent to NT with a neighbourhood partitioning strategy, thereby enabling the practical implementation of NT on standard hardware.
Comprehensive experimental results validate the superior performance of NT on both heterophilic and homophilic graphs when compared to the current state-of-the-art methods.

%
%
%
%
%

\appendix

\section{Hyperparameters}\label{app:params}

\begin{table*}[!t]
\caption{
    Hyperparameters of NT on 10 graphs.
}\label{tbl:params}
\begin{tabular}{lrrrrr}
\toprule
     & Aggregator & \#dimensions & \#heads & \#layers & Dropout \\
\midrule
    Roman Empire            & $\mathrm{sum}$           & 32 & 6 & 5 & 0.4 \\
    A-ratings               & $\mathrm{mean}$          & 40 & 8 & 1 & 0.3 \\
    Minesweeper             & $\mathrm{sum}$           & 53 & 1 & 5 & 0.2 \\
    Tolokers                & gated-$\mathrm{sum}$     & 30 & 2 & 5 & 0.1 \\
    Questions               & $\mathrm{sum}$           & 32 & 4 & 1 & 0.2 \\
    Roman Empire (directed) & $\mathrm{max}$           & 36 & 5 & 5 & 0.4 \\
    A-ratings (directed)    & $\mathrm{max}$           & 23 & 7 & 4 & 0.4 \\
    Minesweeper (directed)  & $\mathrm{sum}$           & 15 & 2 & 5 & 0.1 \\
    Tolokers (directed)     & gated-$\mathrm{sum}$     &  9 & 4 & 4 & 0.2 \\
    Questions (directed)    & gated-$\mathrm{sum}$     & 27 & 7 & 1 & 0.3 \\
    A-computer              & $\mathrm{sum}$           & 17 & 4 & 5 & 0.4 \\
    A-photo                 & $\mathrm{mean}$          & 18 & 7 & 4 & 0.6 \\
    CoauthorCS              & weighted-$\mathrm{mean}$ & 41 & 8 & 2 & 0.3 \\
    CoauthorPhy             & weighted-$\mathrm{mean}$ & 16 & 2 & 2 & 0.1 \\
    WikiCS                  & $\mathrm{mean}$          & 38 & 1 & 3 & 0.2 \\
\bottomrule
\end{tabular}
\end{table*}

\begin{table*}[!t]
\caption{
    Experimental settings of NT on ablation studies.
}\label{tbl:ablation}
\begin{tabular}{lrrrr}
\toprule
    & \Figref{fig:ntmem} & \Figref{fig:frame} & \Figref{fig:agg} & \Figref{fig:selfloops} \\
\midrule
    Aggregator    & $\mathrm{mean}$ & weighted-$\mathrm{mean}$ & |    & weighted-$\mathrm{mean}$\\
    \#dimensions  & 8               & 8                        & 8    & 8 \\
    \#heads       & 4               & 4                        & 8    & 8 \\
    \#layers      & 1               & 1                        & 2    & 2 \\
    Dropout       & 0               & 0.2                      & 0.2  & 0.2 \\
    Learning rate & 0.01            & 0.01                     & 0.01 & 0.01 \\
    \#Epochs      & 500             & 1000                     & 200  & 200 \\
    Early stop    & 50              & 200                      & 200  & 200 \\
\bottomrule
\end{tabular}
\end{table*}

The optimal hyperparameters identified for NT are summarized in Table~\ref{tbl:params}.
The results presented in \Figref{fig:sep} is obtained by exploring the identical hyperparameter space described in \Secref{sec:node}.
Other ablation studies are conducted with manually assigned experimental settings, as detailed in Table~\ref{tbl:ablation}.

\section{Analysis on the Descrepencies of Neighbourhood Sizes in Training and in Inference}\label{app:desc}

\begin{table}[!t]
\caption{
    The descrepencies of neighbourhood sizes in training and in inference.
}\label{tbl:desc}
\begin{tabular}{lr}
\toprule
    Dataset & Descrepency \\
\midrule
    Roman Empire & 3.5\% \\
    A-ratings    & 5.3\% \\
    Minesweeper  & 0.6\% \\
    Tolokers     & 8.8\% \\
    Questions    & 3.9\% \\
    A-computer   & 6.5\% \\
    A-photo      & 8.8\% \\
    CoauthorCS   & 6.5\% \\
    CoauthorPhy  & 7.2\% \\
    WikiCS       & 14.4\% \\
\bottomrule
\end{tabular}
\end{table}

To check if NT performs consistently when neighbourhood sizes are shifted from training to the inference stage, we conduct an analysis to measure the discrepancy between neighbourhood sizes in the training set and beyond.

We first approximate the averaged size of belonging neighbourhoods for each node using $\vs = \text{deg}(\mA^2) / \text{deg}(\mA)$, where $\mA$ is the adjacency matrix and $\text{deg}(\cdot)$ is a function to derive node degrees from the adjacency matrix.
Then, with 100 histogram bins, we transform elements of $\vs$ corresponding to the training set to distribution $P$ and the other elements to distribution $Q$.
After that, we calculate the discrepency between the two distributions as $\sum\limits_i |p_i - q_i|$, where $p_i$ is the probability of the $i$-th histogram bin of $P$ and $q_i$ is that of $Q$.

The discrepencies for the 10 datasets are reported in Table~\ref{tbl:desc}, indicating varying levels of discrepancy across the datasets, with Minesweeper showing low discrepancy and WikiCS showing high discrepancy.
Despite these variations, NT maintains consistent performance, as demonstrated by the experiments in the main text.

\section{Comparing With Graph Transformers}\label{app:gt}

This section presents the comparison between NT and state-of-the-art graph transformers (GT) by referencing the latest data from Polynormer's publication~\cite{Polynormer}, where GraphGPS~\cite{GraphGPS}, NAGphormer~\cite{NAGphormer}, Exphormer~\cite{Exphormer}, NodeFormer~\cite{NodeFormer}, DIFFormer~\cite{DIFFormer}, and GOAT~\cite{GOAT} are also reported.
As discussed in the Related Work section, we emphasize that vanilla GT cannot replace Message Passing Neural Networks (MPNN) as a standalone paradigm, due to their inherent topological information loss.
Consequently, modern GT-based methods are typically integrated with MPNNs (e.g., GraphGPS), and can be more precisely regarded as `GT-augmented MPNNs'.
Against this background, in addition to the vanilla NT proposed in this work, we further develop an enhanced variant NT* by substituting the MPNN component in Polynormer with an NT layer, yielding a `GT-augmented NT' architecture for fair comparison.

\begin{table*}[!t]
\caption{
    Averaged accuracy scores and the standard deviations in 10 runs on heterophilic graphs.
    The best score for each dataset is \textbf{\textcolor{red}{bolded}}, the second best is \underline{\textcolor{blue}{underlined}}, and the third is \textit{\textcolor{olive}{italic}}.
}\label{tbl:hetgt}
\begin{tabular}{lccccc}
\toprule
                & Roman Empire & A-ratings  & Minesweeper & Tolokers   & Questions    \\
\midrule
     GraphGPS   & 82.00±0.61   & 53.10±0.42 & 90.63±0.67  & 83.71±0.48 & 71.73±1.47 \\
     NAGphormer & 74.34±0.77   & 51.26±0.72 & 84.19±0.66  & 78.32±0.95 & 68.17±1.53 \\
     Exphormer  & 89.03±0.37   & 53.51±0.46 & 90.74±0.53  & 83.77±0.78 & 73.94±1.06 \\
     NodeFormer & 64.49±0.73   & 43.86±0.35 & 86.71±0.88  & 78.10±1.03 & 74.27±1.46 \\
     DIFFormer  & 79.10±0.32   & 47.84±0.65 & 90.89±0.58  & 83.57±0.68 & 72.15±1.31 \\
     GOAT       & 71.59±1.25   & 44.61±0.50 & 81.09±1.02  & 83.11±1.04 & 75.76±1.66 \\
     Polynormer & \textit{\textcolor{olive}{92.55±0.37}} & \textbf{\textcolor{red}{54.81±0.49}} & \underline{\textcolor{blue}{97.46±0.36}} & \textbf{\textcolor{red}{85.91±0.74}} & \underline{\textcolor{blue}{78.92±0.89}} \\
     \textbf{NT} & \textbf{\textcolor{red}{94.77±0.31}} & \textit{\textcolor{olive}{54.25±0.50}} & \textit{\textcolor{olive}{97.42±0.50}} & \textit{\textcolor{olive}{85.69±0.54}} & \textit{\textcolor{olive}{78.46±1.10}} \\
     \textbf{NT*} & \underline{\textcolor{blue}{92.90±0.30}} & \underline{\textcolor{blue}{54.55±0.39}} & \textbf{\textcolor{red}{97.92±0.28}} & \textbf{\textcolor{red}{85.91±0.83}} & \textbf{\textcolor{red}{79.36±0.69}} \\
\bottomrule
\end{tabular}
\end{table*}

\begin{table*}[!t]
\caption{
    Averaged accuracy scores and the standard deviations in 10 runs on homophilic graphs.
    The best score for each dataset is \textbf{\textcolor{red}{bolded}}, the second best is \underline{\textcolor{blue}{underlined}}, and the third is \textit{\textcolor{olive}{italic}}.
}\label{tbl:homogt}
\begin{tabular}{lccccc}
\toprule
           & A-computer & A-photo    & CoauthorCS & CoauthorPhy & WikiCS \\
\midrule
GraphGPS   & 91.19±0.54 & 95.06±0.13 & 93.93±0.12 & 97.12±0.19  & 78.66±0.49 \\
NAGphormer & 91.22±0.14 & 95.49±0.11 & \textit{\textcolor{olive}{95.75±0.09}} & \textbf{\textcolor{red}{97.34±0.03}} & 77.16±0.72 \\
Exphormer  & 91.47±0.17 & 95.35±0.22 & 94.93±0.01 & 96.89±0.09  & 78.54±0.49\\
NodeFormer & 86.98±0.62 & 93.46±0.35 & 95.64±0.22 & 96.45±0.28  & 74.73±0.94\\
DIFFormer  & 91.99±0.76 & 95.10±0.47 & 94.78±0.20 & 96.60±0.18  & 73.46±0.56 \\
GOAT       & 90.96±0.90 & 92.96±1.48 & 94.21±0.38 & 96.24±0.24  & 77.00±0.77 \\
Polynormer& \underline{\textcolor{blue}{93.68±0.21}} & \underline{\textcolor{blue}{96.46±0.26}} & 95.53±0.16 & 97.27±0.08 & \underline{\textcolor{blue}{80.10±0.67}}\\
\textbf{NT} & \textit{\textcolor{olive}{92.61±0.63}} & \textit{\textcolor{olive}{96.12±0.39}} & \underline{\textcolor{blue}{96.07±0.32}} & \textit{\textcolor{olive}{97.32±0.11}} & \textit{\textcolor{olive}{80.04±0.61}} \\
\textbf{NT*} & \textbf{\textcolor{red}{93.87±0.23}} & \textbf{\textcolor{red}{96.67±0.17}} & \textbf{\textcolor{red}{96.09±0.11}} & \underline{\textcolor{blue}{97.33±0.09}} & \textbf{\textcolor{red}{80.17±0.71}} \\
\bottomrule
\end{tabular}
\end{table*}

As shown in Table~\ref{tbl:hetgt} and Table~\ref{tbl:homogt}, NT* matches or surpasses all competing GT-based methods on 8 out of 10 evaluated datasets, except A-ratings and CoauthorPhy, demonstrating that the GT+NT composite framework is superior to the prevailing GT+MPNN paradigm.
Moreover, even the vanilla NT outperforms most GT-based methods across all datasets and consistently ranks among the top 3 models.
On heterophilic datasets listed in Table~\ref{tbl:hetgt}, we observe that vanilla NT achieves a substantial performance gain over GT methods on the Roman Empire dataset, owing to its explicit modeling of directional information.
In contrast, standard GT layers ignore node edges and fail to encode edge directionality, limiting their expressiveness on such graph structures.
Across the five homophilic graphs in Table~\ref{tbl:homogt}, the vanilla NT achieves an average rank of $(3+3+2+3+3) / 5 = 2.8$, which is even better than the average rank of $(2+2+5+4+2) / 5 = 3.0$ achieved by Polynormer.
This verifies that the standalone NT layer, even without auxiliary GT components, is highly competitive against state-of-the-art GT+MPNN models.

\section{The Difference Between Graph Partitioning Algorithms and Our Neighbourhood Partitioning}

The primary objective of our algorithm is not to partition a large graph into independent subgraphs for mini-batch training, akin to those graph partitioning algorithms~\cite{METIS, ClusterGCN, GraphSAINT}.
Instead, our algorithm groups graph nodes in a way that still requires full-batch training.
This is because, in NT, applying transformers in all neighborhoods is only an intermediate step; the final representation of a node is then aggregated from all its belonging neighborhoods.

We illustrate this difference with an example.
Suppose a graph partitioning algorithm divides a graph into isolated subgraphs A, B, and C.
A graph model can then be applied to one of these subgraphs, say A, to obtain the representation of a node, say $v$ that is in subgraph A.
However, in NT, even if we also divide graph nodes into groups A, B, and C, the process is different.
If node $v$ in group A has connected nodes $u$ in group B and $w$ in group C, the representation of node $v$ is not derived from applying the transformer in group A.
Instead, it is obtained by gathering information from the neighborhoods of nodes $u$ and $w$ after the transformer has been applied to groups B and C, respectively.
In summary, while the graph nodes are partitioned in NT, all their neighborhoods must still be processed by the transformer, with their results being combined later.

Our partitioning algorithm is designed to optimize both memory and computation time during the transformer's processing, before the results combination.
Thus, our method is tailored only to address the unique challenges that arise in NT and is fundamentally different from graph partitioning algorithms.

\printcredits

\bibliographystyle{cas-model2-names}

\bibliography{nt.bib}



\end{document}